\documentclass{article}
\usepackage{algorithm,algorithmic}
\usepackage[labelfont=it,margin=.5cm,]{caption}
\usepackage{siunitx,  
            booktabs}  

\usepackage{amsmath, amsthm, amssymb, graphicx}
\allowdisplaybreaks

\usepackage{multirow}

\usepackage{times}
\usepackage{sectsty}
\sectionfont{\fontsize{12}{15}\selectfont}
\subsectionfont{\fontsize{10}{13}\selectfont}

\newcommand{\CH}{\text{CH}}
\newcommand{\F}{{\mathcal{F}}}

\newcommand{\zetaB}{\mbox{\boldmath$\zeta$\unboldmath}}

\newcommand{\EBB}{\mathbb{E}}
\newcommand{\PBB}{\mathbb{P}}
\newcommand{\FM}{\mathcal{F}}
\newcommand{\LM}{\mathcal{L}}

\newcommand{\g}{\mathbf{g}}

\newtheorem{theorem}{Theorem} 
\newtheorem{assumption}{Assumption} 
\newtheorem{claim}[theorem]{Claim}

\newtheorem{proposition}[theorem]{Proposition}
\newtheorem{lemma}[theorem]{Lemma}

\newtheorem{definition}[theorem]{Definition}

\def\X{{\mathcal X}}

\def\Y{{\mathcal Y}}
\def\A{{\mathcal A}}

\def\D{{\mathcal D}}
\def\X{{\mathcal X}}

\def\Y{{\mathcal Y}}

\newcommand{\y}{\ensuremath{\mathbf y}}

\newcommand{\K}{\ensuremath{\mathcal K}}

\def\x{\mathbf{x}}
\def\y{\mathbf{y}}

\def\lhat{\hat{\ell}}

\newcommand{\ignore}[1]{}

\DeclareMathOperator*{\argmin}{arg\,min}

\newcommand{\E}{\mathbb{E}}

\newcommand{\reals}{\mathbb{R}}

\renewcommand{\leq}{~\le~}

\let\oldtfrac\tfrac
\renewcommand{\tfrac}[2]{\smash{\oldtfrac{#1}{#2}}}

\let\nablaold\nabla
\renewcommand{\nabla}{\nablaold\mkern-2.5mu}

     \usepackage[nonatbib,preprint]{neurips_2020}

\usepackage[utf8]{inputenc} 
\usepackage[T1]{fontenc}    
\usepackage{hyperref}       
\usepackage{booktabs}       
\usepackage{amsfonts}       
\usepackage{nicefrac}       
\usepackage{microtype}      

\title{Online Boosting with Bandit Feedback}

\author{
  Nataly Brukhim$^{1\,2}$ \hspace{45pt} Elad Hazan$^{1\,2}$\\ 
  $^1$ Department of Computer Science, Princeton University \\
  $^2$ Google AI Princeton \\
\texttt{\{nbrukhim,ehazan\}@princeton.edu}\\
}

\begin{document}

\maketitle

\begin{abstract}


We consider the problem of online boosting for regression tasks, when only limited information is available to the learner. We give an efficient regret minimization method that has two implications: an online boosting algorithm with noisy multi-point bandit feedback, and a new projection-free online convex optimization algorithm with stochastic gradient, that improves state-of-the-art guarantees in terms of efficiency. 
\end{abstract}

\section{Introduction} \label{sec:intro}

Boosting is a fundamental methodology in machine learning which allows us to efficiently convert a number of weak learning rules into a strong one. The theory of boosting in the batch setting has been studied extensively, leading to a tremendous practical success. See \cite{Schapire2012} for a thorough discussion. 

In contrast to the batch setting, online learning algorithms
typically don’t make any stochastic assumptions about the
data. They are often faster, memory-efficient, and can adapt to the best changing predictor over time. 
A line of previous work has explored extensions of boosting methods to the online learning setting \cite{leistner2009robustness, chenonline, chen2014boosting, beygelzimer2015optimal, beygelzimer2015online, agarwal2019boosting, brukhim2020online}.  Of these, several works \cite{beygelzimer2015online, agarwal2019boosting} formally address the setting of online boosting for regression, providing theoretical guarantees on variants of the Gradient boosting method \cite{friedman2001greedy,mason2000boosting} widely used in practice. However, such guarantees are only provided under the assumption that full information is available to the learner, i.e., that the entire loss function is revealed after each prediction is made. 

On the other hand, in many online learning problems, the feedback available to the learner is limited. These problems
 naturally occur in many practical applications, in which interactions with the environment are costly, and the learner has to operate under bandit feedback. Such is often the case, for example, for Reinforcement Learning in a Markov decision process \cite{jin2019learning, rosenberg2019online_b}. In the bandit feedback model, the learner only observes the loss values related to the predictions she chose. In particular, the loss function is not revealed to the learner and, unless the prediction was correct, the true label remains unknown. In this paper we propose the first online boosting algorithm with theoretical guarantees, in the bandit feedback setting.


The underlying ideas used in our approach are based on the fact that boosting can be seen as an optimization procedure. It can be interpreted as cost minimization over the set of linear combinations of weak learners. That is, boosting can be thought of as applying a gradient-descent-type algorithm in a function space 
\cite{Schapire2012, friedman2001greedy, mason2000boosting}. This functional view of boosting has also inspired a few studies of boosting methods \cite{friedman2001greedy, wang2015functional, beygelzimer2015online} that are based on the classical Frank-Wolfe algorithm \cite{frank1956algorithm}, a projection-free convex optimization method. 

In this work we leverage these ideas to yield a new online boosting algorithm based on a Frank-Wolfe-type technique. Namely, our online boosting algorithm is based on a projection-free Online Convex Optimization (OCO) method with stochastic gradients. The stochastic gradient assumption can capture, in particular, bandit feedback, since stochastic gradient estimates can be obtained using random function evaluation \cite{flaxman2005online}.

However, such existing projection-free OCO methods either achieve suboptimal regret bounds \cite{hazan2012projection}
or have high per-iteration computational costs \cite{mokhtari2018stochastic, chen2018projection, xie2019stochastic}. To fill this gap, we derive a new method and analysis of a projection-free OCO algorithm with stochastic gradients. As summarized in Table \ref{table_related_algorithms}, our projection-free OCO algorithm is the fastest known method compared to previous work, while achieving an optimal regret bound. Furthermore, our Frank-Wolfe-type algorithm gives rise to an efficient online boosting method in the bandit setting.

\paragraph{Our results} We propose new online learning methods using only limited feedback. Specifically:
\begin{itemize}
    \item \textbf{Online Boosting with Bandit Feedback}, we propose the first online boosting algorithm with theoretical regret bounds in the bandit feedback setting. The formal description of our method is given in Algorithm \ref{alg:boosting}, and its theoretical guarantees are stated in Theorem \ref{thm:main_boost}. In addition, Section \ref{sec:experiments} presents encouraging experiments on benchmark datasets.

    \item \textbf{Projection-Free OCO with Stochastic Gradients}, an efficient projection-free OCO algorithm, with stochastic gradients, which improves state-of-the-art guarantees in terms of computational efficiency. Table \ref{table_experiments} compares these results to previous work. Our method is given in Algorithm \ref{alg:fw}, and its theoretical guarantees are stated in Theorems \ref{thm:main} and \ref{whp_thm:main}. 

\begin{table*}[h!]
		\caption{Comparison of projection-free Online Convex Optimization methods.}\label{table_related_algorithms}
	\begin{tabular}{@{}p{\textwidth}@{}}
	\centering
			\bgroup
\def\arraystretch{1.5}
	\begin{tabular}{|c|c|c|c|c|}
		\hline
		  \multirow{2}{*}{Algorithm}        & \multirow{2}{*}{Regret} & Per-round      & \multirow{2}{*}{Feedback}     & \multirow{2}{*}{Guarantee}  \\
		         &  &  Cost     &      &  \\
		\hline
		Online-FW \cite{hazan2012projection}  & $O(T^{3/4})$  & $O(1)$ & Full  & deterministic \\
		Meta-FW \cite{chen2018projection}  & $O(\sqrt{T})$  & $O(T^{3/2})$  & Stochastic  &  in expectation  \\
		MORGFW  \cite{xie2019stochastic}  & $\tilde{O}(\sqrt{T})$ &  $O(T)$  & Stochastic  & w.h.p.  \\
		\textbf{This Work} (Thm. \ref{whp_thm:main})   & $\tilde{O}(\sqrt{T})$ &  $O(\sqrt{T})$  & Stochastic  &  w.h.p.  \\
		\hline
	\end{tabular}
	\egroup
	\end{tabular}
\end{table*}

\end{itemize}

\paragraph{Paper outline} In the next subsection we discuss related work. 
Section \ref{sec:fw} deals with the setting of projection-free online convex optimization, with stochastic gradient oracle. We describe the OCO algorithm and formally state its theoretical guarantees. In Section \ref{sec:boosting} we describe a generalization of these techniques, and give our main algorithm of online boosting in the bandit feedback model, along with the main theorem. In Section \ref{sec:experiments} we empirically evaluate the performance of our algorithms. The complete analysis and proofs of all our methods are given in the supplementary material.

\subsection{Related work}

\paragraph{Projection-free OCO.} The classical Frank-Wolfe (FW) method was introduced in \cite{frank1956algorithm} for efficiently solving linear programming. The framework of Online Convex Optimization (OCO) was introduced by \cite{zinkevich2003online}, with the 
 online projected gradient descent method, achieving $O(\sqrt{T})$ regret bound. However, the projections required for such an algorithm are too expensive for many large-scale online problems. The online variant of the FW algorithm that applies to general OCO was given in \cite{hazan2012projection}. It attains $O(T^{3/4})$ regret for the general OCO setting, with only one linear optimization step per iteration. A more general setting considers the use of stochastic gradient estimates instead of exact gradients \cite{mokhtari2018stochastic, chen2018projection, xie2019stochastic}. Although it enables to remove the assumption that exact gradient computation is tractable, it often requires larger computational costs per-iteration. In this work, we give a projection-free OCO method that improves state-of-the-art guarantees with $O(\sqrt{T})$ regret bound, and $O(\sqrt{T})$ per-round cost.

\paragraph{Online Boosting} Previous works on online boosting have mostly focused on classification tasks \cite{leistner2009robustness, chenonline, chen2014boosting, beygelzimer2015optimal, jung2017online, jung2018online}. The main result in this paper is a generalization of the online boosting for regression problems by \cite{beygelzimer2015online}, to the bandit feedback model. We combine these ideas with zero-order convex optimization techniques \cite{flaxman2005online}, and with our novel projection-free OCO algorithm and analysis. Recent works have also considered online boosting in the bandit setting for classification tasks \cite{chen2014boosting, zhang2018online}. These works give convergence guarantees in the more restricted mistake-bound model, whereas in this work we provide regret bounds, compared to a reference function class. The related works of \cite{garber2017efficient, hazan2018online} consider the metric of $\alpha$-regret, which is applicable to computationally-hard problems.


\paragraph{Multi-Point Bandit Feedback} In this work we consider a relaxation of the standard bandit setting:  noisy multi-point bandit feedback. In this model, the learner can query each loss function at multiple points, and obtains noisy feedback values. This model is motivated by reinforcement learning in Markov decision processes.  
Previous work on the multi-point bandit model allows multi-point {\it noiseless} feedback \cite{agarwal2010optimal, duchi2015optimal, shamir2017optimal}. Noiseless feedback is significantly less challenging, since with only two points one can get an arbitrarily good approximation to the gradient.  
In addition, other works have also considered a {\it single point} projection-free noiseless bandit model \cite{garber2019improved,chen2019projection}.

\section{Projection-Free OCO with Limited Feedback} \label{sec:fw}
Consider the setting of Online Convex Optimization (OCO), when only limited feedback is available to the learner, rather than full information. 
Recall that in the OCO framework (see e.g.~\cite{hazan2016introduction}), an online player iteratively makes decisions from a  compact convex set $\mathcal{K} \subset \mathbb{R}^d$. At iteration $t=1,...,T$, the online player chooses $x_t \in \K$, and the adversary reveals the cost $\ell_t$, chosen from $\mathcal{L}$ a family of bounded convex functions  over $\K$. The metric of performance in this setting is regret: the difference between the total loss of the learner and that of the best fixed decision in hindsight. Formally, the regret of the OCO algorithm is defined by:
\begin{equation}\label{oco_regret_def}
    R_{\mathcal{A}}^{\mathcal{L}}(T) = \sum_{t=1}^T \ell_t(x_t) - \underset{x^* \in \mathcal{K}}{\inf} \sum_{t=1}^T \ell_t(x^*).
\end{equation} \\
In this work we restrict the information that the learner has with respect to the loss function $\ell_t$.
Specifically, we focus on two such types of limited feedback:
\begin{enumerate}
    \item \underline{Stochastic Gradients:} the learner is only provided with stochastic gradient estimates. 
    \item \underline{Bandit Feedback:} the learner only observes the loss values of predictions she made.
\end{enumerate}
Our goal is to design an algorithm which has low regret and low cost per iteration $t$. We begin with the more restricted setting which assumes access to a stochastic gradient oracle. In Section \ref{sec:bandit_fw} we describe a reduction for the more general bandit setting, in the context of online boosting. 

As in previous methods of projection-free OCO \cite{mokhtari2018stochastic, chen2018projection, xie2019stochastic}, we assume oracle access to an Online Linear Optimizer (OLO). The OLO algorithm optimizes linear objectives in a sequential manner, and has sublinear regret guarantees. A formal definition is given below. 
\begin{definition} \label{olo}
Let $\LM'$ denote a class of \underline{linear} loss functions, $\ell' : \K \rightarrow \mathbb{R}$, with $\sigma$-bounded gradient norm (i.e., $\|\nabla \ell'(x)\| \le \sigma$). An algorithm $\A$ is an \textbf{Online Linear Optimizer (OLO)} for~$\mathcal{K}$ w.r.t. $\LM'$, if for any sequence $\ell'_1, ...,\ell'_T\in \LM'$, the algorithm has expected regret w.r.t. $\LM'$, $\E[R_{\mathcal{A}}(T,\sigma)]$\footnote{For ease of presentation we denote $R_{\mathcal{A}}(T,\sigma) := R_{\mathcal{A}}^{\mathcal{L'}}(T)$. } that is sublinear in $T$, where expectation is taken w.r.t the internal randomness of $\A$. 
\end{definition}
Suitable choices for the OLO algorithm include Follow the Perturbed Leader (FPL) \cite{kalai2005efficient}, Online Gradient Descent \cite{zinkevich2003online}, Regularized Follow The Leader \cite{hazan2016introduction}, etc. 

Denote the diameter of the set $\K$ by $D > 0$, (i.e., $\forall x, x' \in \mathcal{K}$, $\| x - x' \| \le D$), denote by $G > 0$ an upper bound on the norm of the gradients of $\ell \in \mathcal{L}$ over $\K$ (i.e., $\forall \ell \in \mathcal{L}, x \in \mathcal{K},  \|\nabla \ell(x)\| \le G$), and denote by $M > 0$ an upper bound on the loss (i.e., $\forall \ell \in \mathcal{L}, x \in \mathcal{K}, |\ell(x)| \le M$). We also make the following common assumptions:
\begin{assumption}\label{smooth_grad}
The loss functions $\ell \in \LM$ are $\beta$-smooth, i.e., 
for any $x,x' \in \mathcal{K}$, $\ell \in \LM$, 
$$
\| \nabla \ell(x) - \nabla \ell(x') \| \le \beta\|x - x'\|.
$$
\end{assumption}
\begin{assumption}\label{unbiased_grad}
The stochastic gradient oracle $\mathcal{O}$ returns an unbiased estimate $\g_t = \mathcal{O}(x,t)$, for any $t \in [T], x \in \K$, and with bounded norm, i.e., 
$$
\E[\g_t] = \nabla \ell_t(x) \ \ , \ \ \|\g_{t}\|^2  \leq \sigma^2 .
$$
\end{assumption}

\subsection{Algorithm and Analysis}
At a high level, our algorithm maintains oracle access to $N$ copies of an OLO algorithm, and iteratively 
produces points $x_t$ by running a subroutine of a $N$-step Frank-Wolfe procedure. It uses previous OLOs' predictions, and
gradient estimates oracle in place of exact optimization with true gradients. To update parameters, at each iteration $t$, the algorithm queries the gradient oracle $\mathcal{O}$ at $N$ points. Then, the gradient estimates are fed to the $N$ OLO oracles as linear loss functions. 
Intuitively, it guides each OLO algorithm to correct for mistakes of the preceding OLOs. A formal description is provided in Algorithm \ref{alg:fw}.

\begin{algorithm}[H]
\caption{Projection-Free OCO with Stochastic Gradients Oracle}
\label{alg:fw}
\begin{algorithmic}[1]
\STATE Oracle access: OLO algorithms $\mathcal{A}_1$,...,$\mathcal{A}_N$ (Definition \ref{olo}), and a stochastic gradient oracle $\mathcal{O}$.\\
\STATE Set step length $\eta_i=\frac{2}{i+1}$ for $i \in [N]$.
\FOR{$t = 1, \ldots, T$}
\STATE Define $x_t^0=\mathbf{0}$.
\FOR{$i=1$ to $N$}
\STATE Define $\x_t^i=(1-\eta_i)\x_t^{i-1}+\eta_i \mathcal{A}_i(\g_{1,i},\ldots, \g_{t-1,i})$. 
\STATE Receive stochastic gradient feedback $\g_{t,i} = \mathcal{O}(\x_t^{i-1})$, such that $\E[\g_{t,i}] = \nabla \ell_t(\x_t^{i-1})$.
\STATE Define linear loss function $\ell_t^i(x) =  \g_{t,i}^\top \cdot x$, and pass it to OLO $\mathcal{A}_i$.
\ENDFOR
\STATE Output prediction $x_t:= \x_t^N $.
\STATE Receive loss value $\ell_t(x_t)$.
\ENDFOR
\end{algorithmic}
\end{algorithm}
The following Theorem states the regret guarantees of Algorithm \ref{alg:fw}. In this paper, all bounds are given with respect to the dependence on the different parameters, and omit all constants.
\begin{theorem} \label{thm:main} 
Given that assumptions \ref{smooth_grad} - \ref{unbiased_grad} hold, then Algorithm \ref{alg:fw} is a projection-free OCO algorithm which only requires $N=\frac{\beta D}{\sigma}\sqrt{T}$ stochastic gradient oracle calls per iteration, such that for any sequence of convex losses $\ell_t \in \LM$, and any $x^* \in \mathcal{K}$, its expected regret is, 
	\[
	   \EBB \left[  \sum_{t=1}^T \ell_t(x_t) -  \sum_{t=1}^T \ell_t(x^*) \right] \le O\left(\sigma D\sqrt{T}\right).
\]
\end{theorem} 

The theoretical guarantees given in Theorem \ref{thm:main} use expected regret as the performance metric. Even though expected regret is a widely
accepted metric for online randomized algorithms, one might want to rule out the possibility that the regret has high variance, and verify that the given result actually holds with high probability. 
By observing that excess loss can be formulated as a martingale difference sequence, and by applying analysis using the Azuma-Hoeffding inequality, we can obtain regret guarantees which hold with high probability.  The main result is stated below.

\begin{theorem} \label{whp_thm:main} 
Given that assumptions \ref{smooth_grad} - \ref{unbiased_grad} hold, then Algorithm \ref{alg:fw} is a projection-free OCO algorithm which only requires $N=\frac{\beta D}{\sigma}\sqrt{T}$ stochastic gradient oracle calls per iteration, such that for any $\rho \in (0,1)$, and any sequence of convex losses $\ell_t \in \LM$ over convex set $\mathcal{K}$, w.p. at least $1 - \rho$, 
	\[
	   \sum_{t=1}^T   \ell_t(x_t)  - \underset{x^* \in \mathcal{K}}{\inf} \sum_{t=1}^T \ell_t(x^*) \le O\left( \sigma  D  \sqrt{T \log \frac{\beta D T}{\sigma \rho} } \right).
\] 
\end{theorem}
The complete analysis and proofs of both theorems is deferred to the Appendix. Below we give an overview of the main ideas used in the proof of Theorem \ref{thm:main}. For simplicity assume an oblivious adversary (although using a standard reduction, our results can be generalized to an adaptive one) \footnote{See discussion in \cite{cesa2006prediction}, Pg. 69, as well as Exercise 4.1 formulating the reduction.}. 

Let $\ell_1, ..., \ell_T$ be any sequence of losses in $\mathcal{L}$. Observe that the only sources of randomness at play are: the OLOs' ($\A_i$’s) internal randomness, and the stochasiticity of the gradients. The analysis below is given in expectation with respect to all these random variables. Note the following fact used in the analysis; for any $t,i$, the random variables $\g_{t,i}$ and $\mathcal{A}_i(\g_{1,i},\ldots, \g_{t-1, i})$ (i.e., the output of $\A_i$ at time $t$) are conditionally independent, given all history up to time $t$ and step $i-1$. This fact allows to derive the following Lemma: 
\begin{lemma}\label{lem:expectation}
For any $t \in [T]$ and $i \in [N]$, let $\g_{t,i}$ be the unbiased stochastic gradient estimate used in Algorithm \ref{alg:fw}. Denote the output of algorithm $\A_i$ at time $t$ as $x_{t,i}$. Then, we have,
$$
\E \big[ \ell^i_t( x_{t,i}) \big] = \E \big[ \nabla \ell_t(\x_t^{i-1})^\top \cdot x_{t,i} \big].
$$
\end{lemma}
Using Lemma \ref{lem:expectation}, the algorithm is analyzed along the lines of the Frank-Wolfe algorithm, obtaining the expected regret bound of Algorithm $\ref{alg:fw}$.
\begin{proposition} \label{proposition:main} Given that assumptions \ref{smooth_grad} - \ref{unbiased_grad} hold, and given oracle access to $N$ copies of an OLO algorithm for \textit{linear} losses, with $R_{\mathcal{A}}(T,\sigma)$ regret (see Definition \ref{olo}), Algorithm~\ref{alg:fw} is an online learning algorithm, such that for any sequence of convex losses $\ell_t \in \LM$, and any $x^* \in \mathcal{K}$, its expected regret is, 
	\[
	   \EBB \left[  \sum_{t=1}^T  \ell_t(x_t) - \sum_{t=1}^T \ell_t(x^*) \right] \le \frac{2\beta D^2T}{ N} +  R_{\mathcal{A}}(T,\sigma).
\] 
\end{proposition}

\subsection{Proof of Theorem \ref{thm:main}} \label{thm:main:pf}
\begin{proof}
    The proof of Theorem \ref{thm:main} is a direct Corollary of Proposition \ref{proposition:main}, by plugging \textit{Follow the Perturbed Leader} \cite{kalai2005efficient} as the OLO algorithm required for Algorithm \ref{alg:fw}. We get that the regret of the base algorithms $\A_i$ is $R_{\mathcal{A}}(T,\sigma) = O(\sigma D\sqrt{T})$ w.r.t the sequence of linear losses $\{\ell_t^i\}_t$, where $D$ is the diameter of the set $\mathcal{K}$, and $\sigma$ is the stochastic gradient norm bound (Assumption \ref{unbiased_grad}). Thus, by setting $N = \frac{\beta D}{\sigma}\sqrt{T}$, we get expected regret of $O(\sigma D \sqrt{T})$ w.r.t the convex loss sequence  $\{\ell_t\}_t$.   
\end{proof}

\section{Online Boosting with Bandit Feedback} \label{sec:boosting}
The projection-free OCO method given in Section \ref{sec:fw}, assumes oracle access to an online linear optimizer (OLO), and utilizes it by iteratively making oracle calls with modified objectives, in order to solve the harder task of convex optimization. Analogously, boosting algorithms typically assume oracle access to a "weak" learner, which are utilized by iteratively making oracle calls with modified objective, in order to obtain a "strong" learner, with boosted performance. In this section, we derive an online boosting method in the bandit setting, based on an adaptation of Algorithm \ref{alg:fw}.

In the online learning setting, we assume that in each round $t$ for $t = 1, 2,...,T$, an
adversary selects an example $x_t \in \mathcal{X}$ and a loss function $\ell_t : \Y \rightarrow \mathbb{R}$, where $\Y \subset \mathbb{R}^d$. The loss $\ell_t$ is chosen from a class of bounded convex losses $\LM$. The adversary then presents $x_t$ to the online learning algorithm $\A$, which predicts $\A(x_t)$ in the goal of minimizing the sum of losses over time, when compared against a function class $\F \subset \Y^\X$.
Specifically, the metric of performance in this setting is policy regret: the difference between the total loss of the learner's predictions, and that of the best fixed policy/function $f \in \F$, in hindsight:
\begin{equation} \label{policy_regret_def}
    R_{\mathcal{A}}^{\LM}(T) = \sum_{t=1}^T \ell_t(\A(x_t)) - \underset{f \in \mathcal{F}}{\inf} \sum_{t=1}^T \ell_t(f(x_t)).
\end{equation}
To compare this setting with the OCO setting detailed in Section \ref{sec:fw}, observe that in the OCO setting, at every time step, the adversary only picks the loss function, and the online player picks a point in the decision set $\K$, towards minimizing the loss and competing with the best fixed \underline{point} in hindsight. On the other hand, in this online learning setting, at every time step the adversary picks both an example and a loss function, and the online player picks a point in $\Y$, 
towards minimizing the loss and competing with the best fixed \underline{mapping} in hindsight, of examples in $\X$ to labels in $\Y$. 
Considering these observations, we describe the online boosting methodology next.


Generalizing from the offline setting for boosting, the notion of a weak learning algorithm is modeled as an online learning algorithm
for linear loss functions that competes with a base class of regression functions, while a strong learning algorithm is an online learning algorithm with
 convex loss functions that competes with a larger class of regression functions.  We follow a similar setting to that of the full information Online Gradient Boosting method \cite{beygelzimer2015online}, in the more general case of noisy, bandit feedback, and a weaker notion of weak learner. 

\begin{definition} \label{online_agnostic_wl}
Let $\FM$ denote a reference class of regression functions $f: \X \rightarrow \Y$, let $T$ denote the horizon length, and let $\gamma \ge 1$ denote the advantage.
Let $\LM'$ denote a class of \underline{linear} loss functions, $\ell' : \Y \rightarrow \mathbb{R}$. An online learning algorithm $\A$ is a $(\gamma, T)$-\textbf{agnostic weak online learner (AWOL)} for~$\FM$ w.r.t. $\LM'$, if for any sequence $(x_1, \ell'_1), ...,(x_T, \ell'_T)\in \X \times \LM'$, at every iteration $t \in [T]$, the algorithm outputs $\A(x_t) \in \Y$ such that for any $f \in \FM$, 
\[
\E \Bigg[ \sum_{t=1}^T \ell'_t\big( \A(x_t)\big) - \gamma \text{ } \sum_{t=1}^T \ell'_t\big(f(x_t)\big) \Bigg] \le  R_{\mathcal{A}}(T,\sigma),
\]
where the expectation is taken w.r.t the randomness of the weak learner $\A$ and that of the adversary, and the regret $R_{\mathcal{A}}(T,\sigma)$ is sub-linear in ~$T$. 
\end{definition}
Note the  slight abuse of notation here; $\A(\cdot)$ is not a function but rather the output of the
online learning algorithm $\A$ computed on the given example using its internal state.
Observe that the above definition is the natural extension of the $\gamma$-approximation guarantee of a standard classification weak learner in the statistical setting \cite{Schapire2012}, to regression problems in the online
learning setting. 

The weak learning algorithm is "weak"  in the sense that it is only required to, (a) learn linear loss functions, (b) succeed on full-information feedback, and (c) $\gamma$-approximate the best predictor in its reference class $\FM$, up to an additive regret. Our main result is an online boosting algorithm (Algorithm \ref{alg:boosting}) that converts a weak online learning algorithm, as defined above, into a strong online learning algorithm. The resulting algorithm is "strong" in the sense that it, (a) learns convex loss functions,  (b) relies on bandit feedback only, and (c) $1$-approximates the best predictor in a larger class of functions, $\CH(\FM)$ the convex hull of the base class $\FM$, up to an additive regret. 

\subsection{Setting} \label{sec:bandit_fw_setting} At every round $t$, the learner predicts $y \in \Y$, and receives the noisy bandit feedback $\tilde{\ell}_t(y) = \ell_t(y) + w$, where the noise is drawn i.i.d from a distribution $\D$. We make no distributional assumptions on the noise apart from the fact that it is zero-mean and bounded. Denote the diameter of the set $\Y$ by $D > 0$, (i.e., $\forall y, y' \in \mathcal{Y}$, $\| y - y' \| \le D$), denote by $L > 0$ 
an upper bound on the norm of the gradients of $\ell \in \mathcal{L}$ over $\X$ (i.e., $\forall \ell \in \mathcal{L}, x \in K,  \|\nabla \ell(x)\| \le L$), and denote by $M > 0$ an upper bound on the loss (i.e., $\forall \ell \in \mathcal{L}, y \in \Y, |\ell(x)| \le M$). Denote the bound on the noise by $M$ w.l.o.g. (i.e.,  $|w| \le M$ for all $w \sim \D$). Additionally, assume that the set $\Y$ is endowed with a projection operation, that we denote by $\Pi_\Y$, and satisfies the following properties,
\begin{assumption}\label{proj}
The function $\Pi_\Y: \reals^d \mapsto \Y$  satisfies that 
for any $z \in \mathbb{R}^d$, $\ell \in \LM$, $\ell \big( \Pi_{\Y}(z) \big) \le \ell (z)$.
\end{assumption} 
Consider the following example which demonstrates that Assumption \ref{proj} is in fact a realistic assumption: for any $\mathcal{Y} \subset \mathbb{R}^d$ let the class of loss functions $\LM$ contain losses that are of the form $\ell(y) = \|y - y_t\|^2$ for some $y_t \in \Y$, and let $\Pi_{\Y}(z) \triangleq  \argmin_{y \in \Y} \|z - y \| $ be the Euclidean projection. Indeed, it can be shown that for any $z \in \mathbb{R}^d$, $\|\Pi_{\Y}(z) -  y_t\|^2 \le  \|z -  y_t\|^2$, simply by a generalization of the Pythagorean theorem. \footnote{Moreover, projections according to other distances, that are not the Euclidean distance, can be defined, in particular with respect to Bregman divergences, and an analogue of the generalized Pythagorean theorem remains valid (see e.g., Lemma 11.3 in \cite{cesa2006prediction}). Thus, any class of loss functions that are measuring distance to some $y_t \in \Y$ based on a Bregman divergences, denote $\ell(y) = B_{\mathcal{R}}(y,y_t)$, corresponds to a suitable projection operation, that is simply $\Pi_{\Y}(z) \triangleq  \underset{y \in \Y}{\argmin} B_{\mathcal{R}}(y,z)$.}

\subsection{Stochastic Gradients to Bandit Feedback} \label{sec:bandit_fw}
We build on the techniques shown in Section \ref{sec:fw}, and describe an implementation of the unbiased stochastic gradient oracle, in the bandit setting. Recall that in the bandit feedback model, the only information revealed to  the learner at iteration $t$ is the loss  $\ell_t(x_t)$ at the point $x_t$ that she has chosen. In particular, the learner does not know the loss had she chosen a different point $x_t$. 

We consider a more relaxed noisy multi-point bandit setting, in which the learner can choose several points for which the loss value will be observed. We remark that unlike previous work on multi-point bandit \cite{agarwal2010optimal, duchi2015optimal, shamir2017optimal} we consider noisy feedback, and do not require additional assumptions on the loss function, as we show next. 

The idea is to combine the method in Algorithm \ref{alg:fw}, with gradient estimation techniques for the bandit setting, by \cite{flaxman2005online}. The approach of \cite{flaxman2005online} is based on constructing a simple estimate of the gradient, computed by evaluating the loss $\ell_t$ at a random point. Therefore, we obtain a smoothed approximation of the loss function. Note that since we construct a smoothed approximation of the loss, the smoothness assumption (Assumption \ref{smooth_grad}) becomes redundant, as well the stochastic gradient oracle (Assumption \ref{unbiased_grad}). The following lemmas introduce the smoothed loss function and its properties:
\begin{lemma}[\cite{flaxman2005online},  Lemma 2.1]  \label{lemma:fkm}
Let $\mathcal{L}$ be a set of convex loss functions $\ell : \Y \rightarrow \mathbb{R}$ that are $L$-Lipschitz.
For any $\ell \in \mathcal{L}$, define the function $\lhat \in \hat{\mathcal{L}}$ as follows: $\lhat(y) \triangleq  \E_v[\ell(y+ \delta v)]$, where $v$ is a unit vector drawn uniformly at random, and $\delta >0$. Then, $\lhat$ is differentiable with  gradient: 
$$
\nabla \lhat(y) = \E_v\bigg[\frac{d}{\delta} \ell(y+ \delta v)v\bigg].
$$
\end{lemma}

\begin{lemma} \label{lemma:lhat_prop}
Let $\lhat \in \hat{\mathcal{L}}$, be a smoothed function as defined in Lemma \ref{lemma:fkm}. Then, the following holds:
\begin{enumerate}
    \item $\lhat$ is convex, $L$-Lipschitz, and for any $y \in \mathcal{Y}$, $|\lhat(y) - \ell(y)| \le \delta L$.
    \item For any $y,y' \in \Y $, $\| \nabla \lhat(y) - \nabla \lhat(y') \| \le \frac{d}{\delta}L\|y - y'\|$. Thus, $\lhat$ is $\frac{dL}{\delta}$-smooth.
    \item For any $y \in \Y$, unit vector $v$, $\|  \frac{d}{\delta} \ell(y+ \delta v)v \| \le \frac{dM}{\delta} \triangleq \sigma$.
\end{enumerate}
\end{lemma}


\subsection{Algorithm and Analysis}
At a high level, our boosting algorithm maintains oracle access to $N$ copies of a weak learning algorithm (see Definition \ref{online_agnostic_wl}), and iteratively produces predictions $y_t$, upon receiving an example $x_t$, by running a subroutine of a $N$-step optimization procedure. It generates a randomized gradient estimator $\g_{t,i}$ of  function $\lhat_t(\cdot)$, a smoothed approximation of the loss function $\ell_t(\cdot)$,\footnote{We assume that one can indeed query $\ell_t(\cdot)$ at any point $y + \delta v$. It is w.l.o.g. since a standard technique (see \cite{agarwal2010optimal, hazan2016introduction}) is to simply run the learners $\A_i$ on a slightly smaller set $(1 - \xi)\Y$, where $\xi > 0$ is sufficiently large so that $y + \delta v$ must be in $\Y$. Since $\delta$ can be arbitrarily small, the additional regret/error incurred is arbitrarily small.} as shown in Lemma \ref{lemma:fkm}, and Lemma \ref{lemma:lhat_prop}. The estimator $\g_{t,i}$ is used in place of exact optimization with true gradients. 

To update parameters, the gradient estimates are fed to the $N$ weak learners as linear loss functions. Recall that $\A_i(\cdot)$ is not a function but rather the output of the algorithm $\A_i$ computed on the given example using its internal state, after having observed $\g_{1,i}...\g_{t-1,i}$.
Intuitively, boosting guides each weak learner $\A_i$ to correct for mistakes of the preceding learner $\A_{i-1}$.  The output prediction of the boosting algorithm (Line 13) relies on the projection operation, described in Assumption \ref{proj}. A formal description is provided in Algorithm \ref{alg:boosting}.

\begin{algorithm} [H]
\caption{Online Gradient Boosting with Noisy Bandit Feedback}
\label{alg:boosting}
\begin{algorithmic}[1]
\STATE Maintain $N$ weak learners $\mathcal{A}_1$,...,$\mathcal{A}_N$ (Definition \ref{online_agnostic_wl}). \\
\STATE Input: $\delta > 0$. Set step length $\eta_i=\frac{2}{i+1}$ for $i \in [N]$.
\FOR{$t = 1, \ldots, T$}
\STATE Receive example $x_t$.
\STATE Define $y_t^0=\mathbf{0}$.
\FOR{$i=1$ to $N$}
\STATE Define $y_t^i=(1-\eta_i)y_t^{i-1}+\eta_i \frac{1}{\gamma}\mathcal{A}_i(x_t)$.
\STATE Draw a unit vector $v_t^i$ uniformly at random.
\STATE Receive bandit feedback: $\tilde{\ell}_t(y_t^{i-1} + \delta v_t^i)$.
\STATE Set $\g_{t,i} = \frac{d}{\delta} \tilde{\ell}_t(y_t^{i-1} + \delta v_t^i) v_t^i$.
\STATE Define linear loss function $\ell_t^i(y) =  \g_{t,i}^\top \cdot y$, and pass $(x_t, \ell_t^i(\cdot))$ to weak learner $\mathcal{A}_i$.
\ENDFOR
\STATE Output prediction $y_t:=\Pi_{\Y} \big( y_t^N \big)$.
\STATE Receive bandit feedback $\tilde{\ell}_t(y_t)$.
\ENDFOR
\end{algorithmic}
\end{algorithm}

The following Theorem states the regret guarantees of Algorithm \ref{alg:boosting}. We remark that although it uses expected regret as the performance metric, it can be converted to a guarantee that holds with high probability, with techniques similar to those used to obtain Theorem \ref{whp_thm:main}.

\begin{theorem} \label{thm:main_boost} Given that the setting in \ref{sec:bandit_fw_setting}, and assumption \ref{proj} hold, and given oracle access to $N$ copies of an online \textit{weak} learning algorithms (Definition \ref{online_agnostic_wl}) w.r.t. reference class $\F$ for \textit{linear} losses, with $R_{\mathcal{A}}(T, \sigma)$ regret, then Algorithm~\ref{alg:boosting} is an online learning algorithm w.r.t. reference class $\CH(\F)$ for convex losses $\ell_t$, such that for any $f \in \CH(\F)$, 
    	\[
    	\E[R_{\mathcal{B}}(T)] = \E \Bigg[ \sum_{t=1}^T \ell_t(\y_t) -   \sum_{t=1}^T  \ell_t(f(\x_t))  \Bigg] \le \frac{2d L D^2 T}{\delta \gamma^2 N } + \frac{ R_{\mathcal{A}}(T, dM/\delta)}{\gamma} + 2T\delta L .
\] 
\end{theorem}

Lastly, observe that the average regret $R_{\mathcal{B}}(T)/T$ clearly converges to $0$ as $N \rightarrow \infty$, and $T \rightarrow \infty$. While the requirement that  $N \rightarrow \infty$ may raise concerns about computational
efficiency, this is in fact analogous to the guarantee in the batch setting: the algorithms
converge only when the number of boosting stages goes to infinity. Moreover, previous work on online boosting in the full information setting, gives a lower bound (\cite{beygelzimer2015online}, Theorem 4) which shows that this is indeed necessary.

\begin{table*}
		\caption{Average loss of boosting and baseline algorithms on various datasets, with standard deviation. Relative loss decrease of boosting compared to baseline, shown for bandit setting. }\label{table_experiments}
\begin{tabular}{@{}p{\textwidth}@{}}
	\centering
\begin{tabular}{cccccc} 
\toprule
    & \multicolumn{2}{c}{Full Information} & \multicolumn{2}{c}{Bandit} & \\
    \cmidrule(lr){2-3} \cmidrule(lr){4-5} 
\multirow{2}{*}{Dataset}  & Baseline      & Online   & Baseline       & Online     &  Relative  \\
                   & (OGD) & Boosting  & (N-FKM) & Boosting      &  Decrease      \\\midrule
    abalone & 3.708 \footnotesize{$\pm .027$}& 3.71 \footnotesize{$\pm .006$}& 12.21 \footnotesize{$\pm .210$}& 11.68 \footnotesize{$\pm .154$}& 4.34\% \\
    adult  & 0.154  \footnotesize{$\pm .003$}& 0.151 \footnotesize{$\pm .002$}& 0.161 \footnotesize{$\pm .003$}& 0.150 \footnotesize{$\pm .001$}& 6.83\%  \\
    census  & 0.160 \footnotesize{$\pm .002$} & 0.032 \footnotesize{$\pm .001$}& 0.163 \footnotesize{$\pm .001$}& 0.105 \footnotesize{$\pm .020$}& 35.6\% \\ 
    letter  & 0.507 \footnotesize{$\pm .008$} & 0.498 \footnotesize{$\pm .002$}& 0.522 \footnotesize{$\pm .006$}& 0.517 \footnotesize{$\pm .003$}& 0.95\% \\ 
    slice  & 0.042  \footnotesize{$\pm .0001$} & 0.040 \footnotesize{$\pm .0001$} & 0.049 \footnotesize{$\pm .001$}& 0.045 \footnotesize{$\pm .001$}& 8.16\%  \\ 
    \bottomrule
\end{tabular}
\end{tabular}
\end{table*}

\section{Experiments} \label{sec:experiments}
While the focus of this paper is theoretical investigation
of online boosting and projection-free algorithms with limited information, we have also performed experiments to
evaluate our algorithms. We focused our empirical investigation on the more challenging task of Online Boosting with bandit feedback, proposed in Section \ref{sec:boosting}. Algorithm \ref{alg:boosting} was implemented in NumPy, and the weak online learner was a linear model updated with FKM \cite{flaxman2005online}, online projected gradient descent with spherical gradient estimators. To facilitate a fair comparison to a baseline, we provided an FKM model with a $N$-point noisy bandit feedback, where $N$ is the number of weak learners of the corresponding boosting method. We denote this baseline as N-FKM. We also compare against the full information setting, which amounts to the method used in previous work (\cite{beygelzimer2015online}, Algorithm 2), and compared to a linear model baseline updated with online gradient descent (OGD). Table \ref{table_experiments} summarizes the average squared loss and the standard deviation, and the last column refers to the relative loss decrease on average, of boosting in the bandit setting compared to the N-FKM baseline. 

The experiments we carry out were proposed by \cite{beygelzimer2015online} for evaluating online boosting, they are composed of several data sets for regression and classification tasks, obtained from the UCI machine learning repository (and further described in the supplementary material). For each experiment, reported are average results over 20 different runs. In the bandit setting, each loss function evaluation was obtained with additive noise, uniform on $[\pm .1]$, and gradients were evaluated as in Algorithm \ref{alg:boosting}. The only hyper-parameters tuned were the learning rate, $N$ the number of weak learners, and the smoothing parameter $\delta$. We remark that a small number of weak learners is sufficient, and $N$ was set in the range of $[5,30]$. Parameters were tuned based on progressive validation loss on half of the dataset; reported is progressive validation loss on the remaining
half. Progressive validation is a standard online validation technique, where each training
example is used for testing before it is used for updating the model \cite{blum1999beating}.




\bibliography{bib}
\bibliographystyle{plain}

\appendix

\section{Technical Lemmas}
In this section we give several useful claims and lemmas that are used in the main analysis. 
\begin{lemma} \label{lemma:technical_main}
    Let $\ell: \mathbb{R}^d \rightarrow \mathbb{R}$ be any convex, $\beta$-smooth function. Let $\mathcal{Z} \subset \mathbb{R}^d$ be a set of points with bounded diameter $D$. Let $i \in \mathbb{N}$, and let $z_1,...,z_i \in \mathcal{Z}$. Let $\eta_i \in (0,1)$, and $\gamma \ge 1$.  Define,
    $$
    z^{i} = (1- \eta_i)z^{i-1} - \frac{\eta_i}{\gamma}z_i,
    $$
    and $g_i$ a random variable, such that $\E[g_i] = \nabla \ell(z^{i-1})$. 
    Denote $\zeta_i = (\nabla \ell(z^{i-1}) - g_i )^\top  ( \frac{1}{\gamma} z_i - z)$. Then, for any $z \in \mathcal{Z}$,
    $$
    \Big(\ell(z^{i}) - \ell(z) \Big) \leq\ (1 - \eta_i)\Big(\ell(z^{i-1}) - \ell(z) \Big)  +  \eta_i \Big(g_i^\top  (\frac{1}{\gamma}  z_i - z) + \frac{\eta_i \beta D^2}{2\gamma^2}+  \zeta_{i} \Big) .\\
    $$
\end{lemma}
\begin{proof}
We have,
\begin{align}\label{eq:aaaa}	
		\ell(z^{i})\ &=\ \ell(z^{i-1} + \eta_i(\frac{1}{\gamma}z_i - z^{i-1}))\\
\nonumber		&\leq\  \ell(z^{i-1})  + \eta_i \nabla \ell(z^{i-1})^\top  \cdot (\frac{1}{\gamma}z_i - z^{i-1})  + \frac{\eta_i^2\beta}{2} \| \frac{1}{\gamma}z_i - z^{i-1} \|^2 \\ 
\nonumber		&\leq\ \ell(z^{i-1}) + \eta_i \nabla \ell(z^{i-1})^\top  \cdot (\frac{1}{\gamma} z_i - z^{i-1}) + \frac{\eta_i^2\beta D^2}{2\gamma^2},
	\end{align}
	where the inequalities follow from the $\beta$-smoothness of $\ell$, and the bound on the set $\mathcal{Z}$, respectively. 
	Observe that,
		\begin{align} \label{eq:bbbb}
		\nabla \ell(z^{i-1})^\top   (\frac{1}{\gamma}  z_i - z^{i-1})  &= g_i^\top   ( \frac{1}{\gamma}  z_i - z^{i-1}) + (\nabla \ell(z^{i-1}) - g_i )^\top   ( \frac{1}{\gamma}  z_i - z^{i-1}) \\\nonumber
		& (\text{by adding and subtracting the term: $  g_i^\top   (\frac{1}{\gamma}  z_i - z^{i-1}) $}) \\\nonumber
		&= g_i^\top  ( \frac{1}{\gamma}  z_i - z) +g_i^\top  ( z - z^{i-1}) + (\nabla \ell(z^{i-1}) - g_i )^\top  ( \frac{1}{\gamma}  z_i - z^{i-1}) \\\nonumber
		& (\text{by adding and subtracting the term: $   g_i^\top   z $}) \\\nonumber
		&= g_i^\top  ( \frac{1}{\gamma}  z_i - z) +\nabla \ell(z^{i-1})^\top  ( z - z^{i-1}) + (\nabla \ell(z^{i-1}) - g_i )^\top  ( \frac{1}{\gamma}  z_i - z) \\\nonumber
		& (\text{by adding and subtracting the term:  $ \nabla \ell(z^{i-1})^\top  z $}) \\\nonumber
		&\le g_i^\top  ( \frac{1}{\gamma}  z_i - z) +\ell(z) - \ell(z^{i-1}) + (\nabla \ell(z^{i-1}) - g_i )^\top  ( \frac{1}{\gamma}  z_i - z) \\\nonumber
		& (\text{by convexity, }  \nabla \ell(z^{i-1})^\top  \cdot (z - z^{i-1}) \leq \ell(z) - \ell(z^{i-1})).
\end{align}

	Combining \eqref{eq:aaaa} and \eqref{eq:bbbb}, and the definition of $\zeta_{ i}$ we have that,
	\begin{align*}	
		\Big(\ell(z^{i}) - \ell(z) \Big) \leq\ (1 - \eta_i)\Big(\ell(z^{i-1}) - \ell(z) \Big) + \frac{\eta_i^2\beta D^2}{2\gamma^2} +  \eta_i \Big(g_i^\top  ( \frac{1}{\gamma}  z_i - z) + \zeta_{i} \Big) .\\
	\end{align*}
	
\end{proof}

\begin{claim} \label{induction_claim}
    Define $\eta_i = 2/(i+1)$, for some $i \in \mathbb{N}$. Let $C_1,C_2 > 0$ be some constants, and define $\phi_{i} \in \mathbb{R}$, such that,
    $$
    \phi_{i}\ \le (1-\eta_i) \phi_{i-1}  +\frac{\eta_i^2C_1}{2} + \eta_i C_2.
    $$
Then, it holds that $\phi_{i}\  \le \eta_i C_1 +  C_2.$

\end{claim}
\begin{proof}
    We  prove by induction over $i >0$. For $i = 1$, since $\eta_1 = 1$, the assumption implies that $\phi_{1}\  \le \frac{C_1}{2}+ C_2$. Thus, the base case of the induction holds true. Now assume the claim holds for $i=k$, and we will prove it holds for $i=k+1$. By the induction step,
\begin{align*}
          \phi_{k+1}   &\le \Big(1 - \frac{2}{k+2}\Big)  \phi_{k}  + \frac{2C_1}{(k+2)^{2}} + \frac{2C_2}{k+2}\\
         &\le \frac{k}{k+2} \Big(  \frac{2C_1}{k+1} + C_2\Big)+ \frac{2C_1}{(k+2)^{2}}+ \frac{2C_2}{k+2}\\
         &= \frac{2C_1}{k+2} \Big(  \frac{k}{k+1} +\frac{1}{k+2} \Big)+ C_2 \le \frac{2C_1}{k+2} + C_2.
\end{align*}

\end{proof}




\section{Projection-free OCO with Stochastic Gradients: Proofs}

\subsection{Proof of Lemma \ref{lem:expectation}} \label{sec:lem:expectation:pf}
\begin{proof}
\begin{align*}
     \E \big[ \ell^i_t( x_{t,i}) \big] &= \E \big[ \g_{t,i}^\top \cdot x_{t,i}  \big] \tag{definition of $\ell^i_t(\cdot)$}\\
     &= \underset{\mathcal{I}_t^{i-1}}{\E}\bigg[\E \big[ \g_{t,i}^\top \cdot x_{t,i} \big| \mathcal{I}_t^{i-1}   \big]  \bigg] \tag{law of total expectation}\\
     &\text{( $\mathcal{I}_t^{i-1}$ denotes the $\sigma$-algebra measuring all sources of randomness up to time $t, i-1$.)} \\
     &= \underset{\mathcal{I}_t^{i-1}}{\E}\bigg[\E_{\g_{t,i} } \big[ \g_{t,i}  \big| \mathcal{I}_t^{i-1}   \big]^\top  \cdot\E_{\A_i} \big[ x_{t,i} \big| \mathcal{I}_t^{i-1}   \big]  \bigg] \tag{conditional independence}\\
    &\text{( Inner expectations are w.r.t gradient stochasiticity,} \\
    & \qquad \qquad \text{ and $A_i$'s internal randomness, respectively.)} \\
     &= \underset{\mathcal{I}_t^{i-1}}{\E}\bigg[\nabla \ell_t(\x_t^{i-1})^\top  \cdot\E \big[ x_{t,i} \big| \mathcal{I}_t^{i-1}   \big]  \bigg] \tag{Since $\E[\g_{t,i}] = \nabla \ell_t(\x_t^{i-1})$}\\
     &= \E \big[ \nabla \ell_t(\x_t^{i-1})^\top  \cdot  x_{t,i}   \big]  \\
\end{align*}
\end{proof}

\subsection{Proof of Proposition \ref{proposition:main}} \label{sec:proposition:main:pf}
\begin{proof}
    Let $x_{t,i} \in \mathcal{K}$ be the output of the OLO algorithm $\A_i$ at time $t$, and let $x^*$ be any $\in \mathcal{K}$. 
	The regret definition of $\A_i$ (Definition \ref{olo}), and the definition of $\ell_t^i(\cdot)$ in Algorithm \ref{alg:fw}, imply that:
	\begin{equation} \label{olo_reg_exp}
	\E \Bigg[	\sum_{t=1}^T \g_{t,i}^\top \cdot x_{t,i} \ - \sum_{t=1}^T  \g_{t,i}^\top \cdot x^* \Bigg] \leq\  R_{\mathcal{A}}(T).
	\end{equation}
	By applying Lemma \ref{lemma:technical_main}, we have,
		\[
\Delta_i\  \leq\ (1 - \eta_i)\Delta_{i-1} + \frac{\eta_i^2\beta D^2}{2} T +  \eta_i \sum_{t=1}^T \Big(\g_{t,i}^\top  ( x_{t, i} - x^*) + \zetaB_{t, i} \Big)
\]
where $\Delta_i \triangleq \sum_{t=1}^T \ell_t(\x_t^i) - \ell_t(x^*)$, and $\zetaB_{t, i} \triangleq  (\nabla \ell_t(\x_t^{i-1}) - \g_{t,i} )^\top  \cdot ( x_{t, i} - x^*)$, for $i \in [N]$. 
Take expectation on both sides. By Lemma \ref{lem:expectation}, we have $\E[\zetaB_{t, i}] = 0$, and by the OLO guarantee \eqref{olo_reg_exp}, we get that,
		\[
\EBB\big[\Delta_i\big]  \leq\ (1 - \eta_i)\EBB\big[\Delta_{i-1}\big] + \frac{\eta_i^2\beta D^2}{2} T +  \eta_i R_{\mathcal{A}}(T)
\]
By Claim \ref{induction_claim}, we get for all $i >0$ that,
\begin{equation}\label{eq:thm_induction}
    \EBB\big[\Delta_{i}\ \big] \le \frac{2\beta D^2T}{i+1} +  R_{\mathcal{A}}(T).
\end{equation}
Applying the bound in Equation (\ref{eq:thm_induction}) for $i = N$ concludes the proof.
\end{proof}

\section{High probability bounds for Projection-Free OCO with Stochastic Gradients} \label{sec:appendix_whp}

In this section we give a high-probability regret bound to Algorithm \ref{alg:fw}. Observe that when the variance of the base OLO algorithm is unbounded, the regret guarantees cannot hold with high probability. Thus, we slightly modify the OLO definition to hold w.h.p. This is w.l.o.g as there are projection-free OLO algorithm for which such guarantees hold, as we describe in Theorem \ref{whp_thm:main}.

\begin{definition} \label{whp_olo}
Let $\LM'$ denote a class of \underline{linear} loss functions, $\ell' : \K \rightarrow \mathbb{R}$. An online learning algorithm $\A$ is an \textbf{Online Linear Optimizer (OLO)} for~$\mathcal{K}$ w.r.t. $\LM'$, if for any $\rho \in (0,1)$, and any sequence of losses $\ell'_1, ...,\ell'_T \in \LM'$, w.p. at least $1 - \rho$, the algorithm has regret w.r.t. $\LM'$, $R_{\mathcal{A}}(T)$ that is sublinear in $T$. 
\end{definition}
We can now derive the following proposition (corresponding to Proposition \ref{proposition:main} of the expected case): 
\begin{proposition} \label{whp_proposition:main} Given that assumptions \ref{smooth_grad} - \ref{unbiased_grad} hold, and given oracle access to $N$ copies of an OLO algorithm for \textit{linear} losses, with $R_{\mathcal{A}}(T)$ regret, Algorithm~\ref{alg:fw} is an OCO algorithm which only requires $N = O(\sqrt{T})$ stochastic gradient oracle calls per iteration, such that for any $\rho \in (0,1)$, and any sequence of convex losses $\ell_t$ over convex set $\mathcal{K}$, w.p. at least $1 - \rho$,
	\[
	     \sum_{t=1}^T  \ell_t(x_t) - \underset{x^* \in \mathcal{K}}{\inf} \sum_{t=1}^T \ell_t(x^*) \le \frac{2\beta D^2T}{ N} + R_{\mathcal{A}}(T) + (\sigma +G) D \sqrt{2T\log(4N/\rho)}.
\]
\end{proposition}

\begin{proof}
    Let $x_{t,i} \in \mathcal{K}$ be the output of the OLO algorithm $\A_i$ at time $t$, and let $x^*$ be any point in $\mathcal{K}$. 
	The regret definition of $\A_i$ (Definition \ref{whp_olo}), and the definition of $\ell_t^i(\cdot)$ in Algorithm \ref{alg:fw}, imply that for $\rho \in (0,1)$ we have that, w.p. at least $1 - \rho/(2N)$,
	\begin{equation} \label{eq:whp_olo}
	\sum_{t=1}^T \g_{t,i}^\top \cdot x_{t,i} \ - \sum_{t=1}^T  \g_{t,i}^\top \cdot x^*  \leq\  R_{\mathcal{A}}(T).
	\end{equation}

	By applying Lemma \ref{lemma:technical_main}, and by the OLO guarantee \eqref{eq:whp_olo}, we get that,
	\begin{align}	\label{eq:whp_fst_eq}
		\Delta_i\  \leq\ (1 - \eta_i)\Delta_{i-1} + \frac{\eta_i^2\beta D^2}{2} T +  \eta_i \Big( R_{\mathcal{A}}(T) +  \sum_{t=1}^T \zetaB_{t, i} \Big) .
	\end{align}
	where $\Delta_i \triangleq \sum_{t=1}^T \ell_t(\x_t^i) - \ell_t(x^*)$, and $\zetaB_{t, i} \triangleq  (\nabla \ell_t(\x_t^{i-1}) - \g_{t,i} )^\top  \cdot ( x_{t, i} - x^*)$, for $i \in [N]$. By applying the union bound, the above inequality holds for all $i \in [N]$, with probability at least $1 - \rho/2$.\\
	
	For any fixed $i \in [N]$, Observe that $\E\left[\zetaB_{t, i} | \mathcal{I}^i_{t-1}\right] = \EBB\left[(\nabla \ell_t(\x_t^{i-1}) - \g_{t,i} )^\top  \cdot ( x_{t, i} - x^*) | \mathcal{I}^i_{t-1}\right] = 0$ by Lemma \ref{lem:expectation}. Therefore, $\{\zetaB_{t, i}\}_{t=1}^T$ is a martingale difference sequence. Moreover, by the Cauchy-Schwartz inequality, we have,
$$
| \zetaB_{t, i} | \le \| \nabla \ell_t(\x_t^{i-1}) - \g_{t,i} \| \cdot \| x_{t, i} - x^* \| \le  \Big(\| \nabla \ell_t(\x_t^{i-1})\| +\| \g_{t,i} \| \Big)\cdot \| x_{t, i} - x^* \| \le (G + \sigma) \cdot D = c_t,
$$
where the second inequality follows from the triangle inequality, and the last inequality follows from the diameter bound $D$ on the set $\K$, the bound on the gradient norm $G$, and the bound on the stochastic gradient estimate (Assumption \ref{unbiased_grad}). Let $\lambda = (\sigma + G) D \sqrt{2T\log(4N/\rho)}$, by the Azuma-Hoeffding inequality, 
\begin{equation*}
    \PBB \Bigg[ \Big| \sum_{t = 1}^T \zetaB_{t, i} \Big| \ge \lambda \Bigg]
    \le 2 \text{ }\mathrm{ exp}\left(- \frac{\lambda^2}{2 \sum_{t = 1}^T c_{t}^2}\right) = \rho / 2N.
\end{equation*}

Observe that, by applying the union bound, the above inequality holds for all $i \in [N]$, with probability at least $1 - \rho/2$. 
Therefore, by combining the above with \eqref{eq:whp_fst_eq}, applying union bound, we get that w.p. at least $1 - \rho$, we have for all $i \in [N]$, 
	\begin{align*}	
		\Delta_i\  \leq\ (1 - \eta_i)\Delta_{i-1} + \frac{\eta_i^2\beta D^2}{2} T +  \eta_i \Big( R_{\mathcal{A}}(T) +  (\sigma + G) D \sqrt{2T\log(4N/\rho)} \Big) .\\
	\end{align*}

	Applying Claim \ref{induction_claim}, and setting $i = N$ yields that,
\begin{align} \label{eq:whp_N_bound} 
     \sum_{t=1}^T \ell_t(x_t) - \ell_t(x^*)  \le \frac{2\beta D^2T}{ N} + R_{\mathcal{A}}(T) + (\sigma + G) D \sqrt{2T\log(4N/\rho)} .
\end{align}

\end{proof}

\subsection{Proof of Theorem \ref{whp_thm:main}} \label{whp_thm:main:pf}
\begin{proof}
    The proof of Theorem \ref{whp_thm:main} is a direct Corollary of Proposition \ref{whp_proposition:main}, by plugging \textit{Follow the Perturbed Leader} \cite{kalai2005efficient} with high probability guarantees (e.g., \cite{neu2016importance}) as the OLO algorithm required for Algorithm \ref{alg:fw}. We get that the regret of the base algorithms $\A_i$ is $R_{\mathcal{A}}(T) = O(\sigma D\sqrt{T})$, where $D$ is the diameter of the set $\mathcal{K}$, and $\sigma$ is the bound on the stochastic gradient norm (Assumption \ref{unbiased_grad}). Thus, by setting $N = \frac{\beta D}{\sigma}\sqrt{T}$, we get that w.p. at least $1 - \rho$,
    \begin{align*} 
     \sum_{t=1}^T \ell_t(x_t) - \ell_t(x^*)  &\le 2 \sigma D \sqrt{T} + O(\sigma D\sqrt{T}) + (\sigma + G) D  \sqrt{2T\log\big(\beta D T / (\sigma \rho) \big)} \\
     &= O\Big(\sigma  D  \sqrt{T\log\big(\beta D T / (\sigma \rho) \big)} \Big).
\end{align*}
\end{proof}

\subsection{Proof of Lemma \ref{lemma:lhat_prop}} \label{proof_lhat_prop}
\begin{proof} Below are the proofs of each item:
\begin{enumerate}
    \item The fact that  $\lhat$ is convex, $L$-Lipschitz is immediate from its definition
and the assumptions on $\ell$. The inequality follows from $v$ being a unit vector and that $\ell$ is assumed to be $L$-Lipschitz.
    \item For any $x,x' \in \mathbb{R}^d $, \begin{equation*}
    \begin{aligned}[b]
    \| \nabla \lhat(x) - \nabla \lhat(x') \| &= \frac{d}{\delta} \|\EBB[(\ell(x+\delta v) - \ell(x'+\delta v))v] \|  \\
     &\le \frac{d}{\delta} \EBB\Big[\|(\ell(x+\delta v) - \ell(x'+\delta v))v \|\Big]\\
     &\le \frac{d}{\delta} \EBB\Big[|\ell(x+\delta v) - \ell(x'+\delta v) |\Big]\\
    &\le \frac{d}{\delta}L\|x - x'\|,
\end{aligned}
    \end{equation*}
    where the first inequality follows from Jensen's Inequality, the second inequality follows from the fact that $v$ is a unit vector, and the next inequality from  $\ell$ being $L$-Lipschitz. This property implies that the function $\lhat$ is $\frac{dL}{\delta}$-smooth. 
        \item For any $x \in \mathbb{R}^d $, and unit vector $u$,
    $\| \nabla \lhat(x) - \big(\frac{d}{\delta} \tilde{\ell}(x+ \delta u)u\big)\| \le \|\frac{d}{\delta} \tilde{\ell}(x+ \delta u)u\| + \|\nabla \lhat(x) \| $.
    Note that by the fact that $\lhat$ is $L$-Lipschitz, we have $ \|\nabla \lhat(x) \| \le L$. The first term can be bounded as follows:
    \begin{equation*}
    \begin{aligned}[b]
    \|\frac{d}{\delta} \tilde{\ell}(x+ \delta u)u\| &= \frac{d}{\delta} \tilde{\ell}(x+ \delta u) \| u\| 
    \le \frac{d}{\delta} \tilde{\ell}(x+ \delta u) \\
    &= \frac{d}{\delta} \big( \ell(x+ \delta u) + w \big)
    \le \frac{d}{\delta} 2M,
\end{aligned}
    \end{equation*}
        where the first inequality follows from the fact that $u$ is a unit vector, the equality follows from the definition of $\tilde{\ell}$, and the last inequality follows from the bounds on $\ell$ and $w \sim \D$. Therefore, we have,
        $$
        \| \nabla \lhat(x) - \big(\frac{d}{\delta} \tilde{\ell}(x+ \delta u)u\big)\| \le \frac{2dM}{\delta} +L.
        $$
\end{enumerate}
\end{proof}

\begin{theorem} \label{whp_thm:bandit} 
Algorithm \ref{alg:fw} is a projection-free OCO algorithm for the bandit setting, with $N = \sqrt{T}$ bandit feedback values per round, such that for any $\rho \in (0,1)$, and any sequence of convex losses $\ell_t \in \LM$ over convex set $\mathcal{K}$, w.p. at least $1 - \rho$, 
	\[
	   \sum_{t=1}^T   \ell_t(x_t)  - \underset{x^* \in \mathcal{K}}{\inf} \sum_{t=1}^T \ell_t(x^*) \le O\left(dMLD^2T^{3/4}\sqrt{\log(T/\rho)}\right) = \tilde{O}(T^{3/4}).
\]
\end{theorem}
\subsection{Proof of Theorem \ref{whp_thm:bandit}} \label{whp_thm:bandit:pf}

\begin{proof}
    Observe that by Lemma \ref{lemma:fkm}, we have that Assumptions \ref{smooth_grad}-\ref{unbiased_grad} are redundant, and so Lemma \ref{lem:expectation} and Proposition \ref{whp_proposition:main} hold for losses $\lhat_t \in \hat{\mathcal{L}}$, with $G = L$, $\sigma = dM/\delta$, and  $\beta = dL/\delta$, by Lemma \ref{lemma:lhat_prop}. Thus, we have that w.p at least $1 - \rho$,

\begin{align} \label{eq:whp_N_bound_bandit} 
     \sum_{t=1}^T \lhat_t(x_t) - \lhat_t(x^*)  \le \frac{2d L D^2T}{\delta N} + R_{\mathcal{A}}(T) + (dM/\delta + L) D \sqrt{2T\log(4N/\rho)} .
\end{align}
Now, observe that,
\begin{align} 
          \sum_{t=1}^T \ell_t(x_t) - \ell_t(x^*) &\le   \sum_{t=1}^T \lhat_t(x_t) - \lhat_t(x^*)  + 2T\delta L \tag{By Lemma \ref{lemma:lhat_prop} (1)}\\
         &\le \frac{2d L D^2T}{\delta N} + R_{\mathcal{A}}(T) + (dM/\delta + L) D \sqrt{2T\log(4N/\rho)} + 2T\delta L \tag{By \eqref{eq:whp_N_bound_bandit}} \\
         &\le \frac{2d L D^2T}{\delta N} + O(dMD\sqrt{T}/\delta) + (dM/\delta + L) D \sqrt{2T\log(4N/\rho)} + 2T\delta L \label{eq:whp_bandit_last}
\end{align}

where the last inequality follows by plugging \textit{Follow the Perturbed Leader} \cite{kalai2005efficient} with high probability guarantees (e.g., \cite{neu2016importance}) as the OLO algorithm required for Algorithm \ref{alg:fw}. We get that the regret of the base algorithms $\A_i$ is $R_{\mathcal{A}}(T) = O(dMD\sqrt{T}/\delta)$. \\

Lastly the results follows by plugging in $\delta = T^{-1/4}$ and $N=\sqrt{T}$ into Equation \eqref{eq:whp_bandit_last}, to obtain regret of at most 
$O\left(dMLD^2T^{3/4}\sqrt{\log(T/\rho)}\right) = \tilde{O}(T^{3/4})$, w.p at least  $1 - \rho$.
\end{proof}

\section{Online Boosting: Proofs}
In this section we give the full analysis of the Algorithm and results given in Section \ref{sec:boosting}.
For simplicity assume an oblivious adversary (can also be shown to hold for an adaptive one). Let $(x_1 , \ell_1), ..., (x_T , \ell_T)$ be \textbf{any} sequence of  examples and losses. Observe that the only sources of randomness at play are: the weak learners' ($\A_i$’s) internal randomness, the random unit vectors $v_t^i$, and the additive zero-mean noise for any bandit feedback. The analysis below is given in expectation with respect to all these random variables. \\

\begin{lemma}\label{lem:independence}
For any $t \in [T]$ and $i \in [N]$, let $\mathbf{g}_{t,i}$ be the stochastic gradient estimate used in Algorithm 1, s.t. $\E[\mathbf{g}_{t,i}] = \nabla \lhat(y_t^{i-1})$, and $\ell_t^i(\mathbf{y}) =  \mathbf{g}_{t,i}^\top \cdot \mathbf{y}$. Then, we have,
$$
\E \Big[ \ell^i_t\big( \A_i(x_t)\big) \Big] = \E \Big[ \nabla \lhat(y_t^{i-1})^\top \cdot \A_i(x_t) \Big].
$$
\end{lemma}

\begin{proof}
Let $\mathcal{I}_t^{i-1}$ denotes the $\sigma$-algebra measuring all sources of randomness up to time $t$ and learner $i-1$; i.e., the internal randomness of weak learners $\A_1,...,\A_{i-1}$, the the random unit vectors $v_1^j,...,u_t^j$, for all $j < 1$, and the noise terms $w_{1,j},...,w_{t,j}$ for all $j < 1$. Then,
\begin{align*}
     \E \big[ \ell^i_t\big( \A_i(x_t)\big) \big] &= \E \Big[ \mathbf{g}_{t,i}^\top \cdot \A_i(x_t)  \Big] \tag{definition of $\ell^i_t(\cdot)$}\\
     &= \E \Big[ \Big(\frac{d}{\delta}\tilde{\ell}_t(y_t^{i-1} + \delta v_t^i)\cdot v_t^i \Big)^\top \cdot \A_i(x_t)  \Big] \tag{definition of $\mathbf{g}_{t,i}$}\\
     &= \E \Big[ \Big(\frac{d}{\delta}\ell_t(y_t^{i-1} + \delta v_t^i)\cdot v_t^i \Big)^\top \cdot \A_i(x_t)  \Big] \tag{since $\tilde{\ell}(z) = \ell(z) + w$, with $w$ i.i.d., $\E[w]=0$}\\
     &= \underset{\mathcal{I}_t^{i-1}}{\E}\Bigg[\E \Big[ \Big(\frac{d}{\delta}\ell_t(y_t^{i-1} + \delta v_t^i)\cdot v_t^i \Big)^\top \cdot \A_i(x_t) \Big| \mathcal{I}_t^{i-1}   \Big]  \Bigg] \tag{by law of total expectation}\\
     &= \underset{\mathcal{I}_t^{i-1}}{\E}\Bigg[\E_{v_t^i} \Big[ \frac{d}{\delta}\ell_t(y_t^{i-1} + \delta v_t^i)\cdot v_t^i \Big| \mathcal{I}_t^{i-1}   \Big]^\top  \cdot\E \Big[ \A_i(x_t) \Big| \mathcal{I}_t^{i-1}   \Big]  \Bigg] \tag{by conditional independence}\\
     &= \underset{\mathcal{I}_t^{i-1}}{\E}\Bigg[\nabla \lhat_t(y_t^{i-1})^\top  \cdot\E \Big[ \A_i(x_t) \Big| \mathcal{I}_t^{i-1}   \Big]  \Bigg] \tag{by Lemma \ref{lemma:fkm}}\\
     &= \E \Big[ \nabla \lhat_t(y_t^{i-1})^\top  \cdot  \A_i(x_t)   \Big]  \\
\end{align*}
\end{proof}

\subsection{Proof of Theorem \ref{thm:main_boost}}

\begin{proof}
	First, note that for any $i = 1, 2, \ldots, N$, since $\ell_t^i$ is a linear function, we have
	\[ \inf_{f \in \CH(\F)} \sum_{t=1}^T \ell_t^i(f(\x_t))\ =\ \inf_{f \in \F} \sum_{t=1}^T \ell_t^i(f(\x_t)).  \]
	Let $f$ be any function in $\CH(\F)$. The equality above, the regret bound of the weak learner $\A^i$ for $\F$ (see Definition \ref{online_agnostic_wl}), and the definition of $\ell_t^i(\cdot)$ in Algorithm \ref{alg:boosting}, imply that:
	\begin{equation}  \label{wl_reg_exp}
	\E \Bigg[	\sum_{t=1}^T \g_{t,i}^\top \cdot \A^i(\x_t)\ -\gamma \sum_{t=1}^T  \g_{t,i}^\top \cdot f(\x_t) \Bigg] \leq\  R_{\mathcal{A}}(T).
	\end{equation}

	Now define, for $i = 0, 1, 2, \ldots, N$, $\Delta_i = \sum_{t=1}^T \lhat_t(\y_t^i) - \lhat_t(f(\x_t))$. 
	By applying Lemma \ref{lemma:technical_main}, we get, 
	$$
	\Delta_i\  \leq (1 - \eta_i)\Delta_{i-1}  + \frac{\eta_i^2\beta D^2 }{2\gamma^2} T + \eta_i \sum_{t=1}^T \Big( \g_{t,i}^\top (\frac{1}{\gamma}\A^i(\x_t) - f(\x_t)) +   \zetaB_{t, i} \Big)
	$$
	where $\zetaB_{t, i} \triangleq  (\nabla \ell_t(\y_t^{i-1}) - \g_{t,i} )^\top  \cdot ( \A^i(\x_t) - f(\x_t))$. 
	Take expectation on both sides. By Lemma \ref{lem:independence}, we have $\E[\zetaB_{t, i}] = 0$, and by the weak learning guarantee \eqref{wl_reg_exp}, we get that,
		\[
\EBB\big[\Delta_i\big]  \leq\ (1 - \eta_i)\EBB\big[\Delta_{i-1}\big] + \frac{\eta_i^2\beta D^2}{2\gamma^2} T +  \frac{\eta_i}{\gamma}R_{\mathcal{A}}(T)
\]

	By Claim \ref{induction_claim} (with $\phi_i = \EBB[\Delta_{i}]$), we get,
\begin{equation}\label{eq:induction_boost}
    	\EBB\big[\Delta_{i}\ \big]\le  \frac{2\beta D^2 T}{\gamma^2(i+1)} + \frac{ R_{\mathcal{A}}(T)}{\gamma}.
\end{equation}

Lastly, observe that, 
\begin{align*} \label{eq:gap_and_proj}
         \EBB \Bigg[ \sum_{t=1}^T \ell_t(\y_t) - \ell_t(f(\x_t))\Bigg] &\le \EBB \Bigg[  \sum_{t=1}^T \lhat_t(\y_t) - \lhat_t(f(\x_t)) \Bigg] + 2T\delta L \tag{by Lemma \ref{lemma:lhat_prop} (2)}\\
         &\le \EBB \Bigg[ \sum_{t=1}^T  \lhat_t(\y_t^N) - \lhat_t(f(\x_t)) \Bigg] + 2T\delta L \tag{by Assumption \ref{proj}} \\
         &\le \frac{2\beta D^2 T}{\gamma^2 N } + \frac{ R_{\mathcal{A}}(T)}{\gamma} + 2T\delta L \tag{by (\ref{eq:induction_boost}), for $i=N$}\\
         &\le \frac{2d L D^2 T}{\delta \gamma^2 N } + \frac{ R_{\mathcal{A}}(T)}{\gamma} + 2T\delta L \tag{by Lemma \ref{lemma:lhat_prop} (3)}
\end{align*}
\end{proof}

\section{Experimental setup description}

The datasets were taken from the UCI machine learning repository, and their statistics are detailed below, along with the link to a downloadable version of each dataset.

\begin{small}
\begin{tabular}{@{}p{\textwidth}@{}}
	\centering
\begin{tabular}{crcccc}
\toprule
Dataset & \#Instances & \#Features & Downloadable & Task & Label \\
        &  &  & version & & range \\
\midrule
abalone & 4,177 & 10 & \href{http://archive.ics.uci.edu/ml/datasets/Abalone?pagewanted=all}{Link} & regression & $[1,29]$
\\
adult &  48,842  &  105 & \href{http://archive.ics.uci.edu/ml/datasets/Adult}{Link} & classification & $[0,1]$
\\
census &  299,284  & 401 &  \href{https://archive.ics.uci.edu/ml/datasets/Census-Income+\%28KDD\%29}{Link} & classification & $[0,1]$
\\
letter &  20,000  &  16 & \href{https://archive.ics.uci.edu/ml/datasets/Letter+Recognition}{Link} & classification & $[-1,1]$ 
\\
slice & 53,500 & 385 & \href{https://archive.ics.uci.edu/ml/datasets/Relative+location+of+CT+slices+on+axial+axis}{Link} & regression & $[0,1]$
\\
\bottomrule
\end{tabular}
\end{tabular}
\end{small}
\vspace{2mm}

Algorithm \ref{alg:boosting} was implemented in NumPy, and the weak online learner was a linear model updated with FKM \cite{flaxman2005online}, online projected gradient descent with spherical gradient estimators. To facilitate a fair comparison to a baseline, we provided an FKM model with a $N$-point noisy bandit feedback, where $N$ is the number of weak learners of the corresponding boosting method. We denote this baseline as N-FKM. We also compare against the full information setting, which amounts to the method used in previous work (\cite{beygelzimer2015online}, Algorithm 2), and compared to a linear model baseline updated with online gradient descent (OGD). \\

The experiments we carry out were proposed by \cite{beygelzimer2015online} for evaluating online boosting, they are composed of several data sets for regression and classification tasks, obtained from the UCI machine learning repository. For each experiment, reported are average results over 20 different runs. In the bandit setting, each loss function evaluation was obtained with additive noise, uniform on $[\pm .1]$, and gradients were evaluated as in Algorithm \ref{alg:boosting}. The only hyper-parameters tuned were the learning rate, $N$ the number of weak learners, and the smoothing parameter $\delta$:
\begin{itemize}
    \item $N$ was set in the range of $[5,30]$.
    \item $\delta$ was set to $1/2$ in all the experiments.
    \item Learning rate at time $t$ is lr $ * t^{-c}$ where lr and $c$ were set in the ranges [1e-04$, 0.1]$, $[.25, 1]$. 
\end{itemize}
Parameters were tuned based on progressive validation loss on half of the dataset; reported is progressive validation loss on the remaining
half.

\end{document}